\documentclass[submission,copyright,creativecommons]{eptcs}
\providecommand{\event}{ICLP 2022} % Name of the event you are submitting to
\usepackage{underscore}           % Only needed if you use pdflatex.

\usepackage{ifthen}
\usepackage{url}
\usepackage{amssymb}
\usepackage{amsmath}
\usepackage{import}
\usepackage{xparse}
\usepackage{amsmath}
\usepackage[usenames,dvipsnames]{xcolor}
\usepackage{xspace}
\usepackage{datatool}
\usepackage{glossaries}

\providecommand\m[1]{\ensuremath{#1}\xspace}
\renewcommand{\m}[1]{\ensuremath{#1}\xspace}
\newcommand{\trval}[1]{\m{\mathbf{#1}}}

	\newcommand{\lrule}{\leftarrow}
	\newcommand{\cause}{\stackrel{c}{\lrule}}
	
	\newcommand{\ltrue}{\trval{t}}
	\newcommand{\lfalse}{\trval{f}}
	\newcommand{\lunkn}{\trval{u}}
	
	\newcommand{\Tr}{\ltrue}
	\newcommand{\Fa}{\lfalse}
	\newcommand{\Un}{\lunkn}

	\newcommand{\struct}{\m{I}}

	\newcommand{\rules}{\m{R}}
	
	\NewDocumentCommand\inter{g+g}{%
	  \IfNoValueTF{#1}
	    {\struct}
	    {\m{#1^{#2}}}}

	\renewcommand{\int}{\m{\mathbb{Z}}}

	\newcommand{\leqt}{\m{\leq_t}}
	\newcommand{\geqt}{\m{\geq_t}}
	
	\newcommand{\val}{\m{\nu}}

	\NewDocumentCommand\subs{g+g}{%
	  \IfNoValueTF{#1}
	    {\m{/}}
	    {\m{#1/ #2}}}

\newcommand{\ouracronym}[3]{%
	\newacronym{#1}{#2}{#3}
	\expandafter\newcommand\csname #1\endcsname{\gls{#1}\xspace}%
}
	\ouracronym{FO}{FO}{first-order logic}
	\ouracronym{PC}{PC}{propositional calculus}
	\ouracronym{MX}{MX}{Model Expansion}
	\ouracronym{MO}{MO}{Model Optimization}
	\ouracronym{ASP}{ASP}{Answer Set Programming}
	\ouracronym{TP}{TP}{Theorem Proving}
	\ouracronym{LP}{LP}{Logic Programming}
	\ouracronym{CP}{CP}{Constraint Programming}
	\ouracronym{FP}{FP}{Functional Programming}
	\ouracronym{KR}{KR}{Knowledge Representation}
	\ouracronym{CSP}{CSP}{Constraint Satisfaction Problem}
	\ouracronym{SMT}{SMT}{SAT Modulo Theories}
	\ouracronym{KBS}{KBS}{knowledge base system}
	\ouracronym{NNF}{NNF}{Negation Normal Form}
	\ouracronym{FNNF}{FNNF}{Flat Negation Normal Form}
	\ouracronym{DefNNF}{DefNNF}{Definition Negation Normal Form}
	\ouracronym{DEFNF}{DEFNF}{Definition Normal Form}
	\ouracronym{CDCL}{CDCL}{Conflict-Driven Clause-Learning}
	\ouracronym{WFS}{WFS}{Well-Founded Semantics}
	\ouracronym{LCG}{LCG}{Lazy Clause Generation}
	\ouracronym{AEL}{AEL}{Autoepistemic Logic}
	\ouracronym{OEL}{OEL}{Ordered Epistemic Logic}
	\ouracronym{AFT}{AFT}{Approximation Fixpoint Theory}

	\makeatletter
	\def\ifenv#1{
	\def\@tempa{#1}%
	\def\@ttempa{#1*}%
	\ifx\@tempa\@currenvir
	\expandafter\@firstoftwo
	\else
	\expandafter\@secondoftwo
	\fi
	}
	\makeatother

	\newcommand{\ddrule}[4]{\ensuremath{#1 \leftarrow #2 & \{#3\} & #4}}
	\newcommand{\drule}[2]{\ensuremath{#1 & \leftarrow & #2}}

	\newcommand{\darule}[4]{\ensuremath{#1 \leftarrow #2 & \{#3\} & #4}}
	\newcommand{\arule}[2]{\ensuremath{#1 \, &\leftarrow \, #2}}

	\newcommand{\LNDRule}[2]{
	\ifenv{array}
	{\drule{#1}{#2}}
	{ \ifenv{align}
		{\arule{#1}{#2}}
		{\ifenv{align*}
		{\arule{#1}{#2}}
		{ERROR: using LDRule in unsupported environment: \@currenvir}
		}
	}
	}

	\newcommand{\LDRule}[4]{
	\ifenv{array}
	{\ddrule{#1}{#2}{#3}{#4}}
	{ \ifenv{align}
		{\darule{#1}{#2}{#3}{#4}}
		{\ifenv{align*}
		{\darule{#1}{#2}{#3}{#4}}
		{ERROR: using LDRule in unsupported environment: \@currenvir}
		}
	}
	}

	\NewDocumentCommand\LRule{m+g+g+g}{%
		\IfNoValueTF{#2}%
		{#1.&}{%
		\IfNoValueTF{#3}
		{\LNDRule{#1}{#2.}}
		{\LDRule{#1}{#2.}{#3}{#4}}%
		}
	}

	\NewDocumentCommand\CLRule{m+g}{%
	\ifenv{array}
	{\cdrule{#1}{#2}}
	{ \ifenv{align}
		{\carule{#1}{#2}}
		{\ifenv{align*}
			{\carule{#1}{#2}}
			{ERROR: using CLRule in unsupported environment: \@currenvir}
		}
	}
	}

	\NewDocumentCommand\carule{m+g}{%
		\IfNoValueTF{#2}
			{\ensuremath{#1.}}
			{\ensuremath{#1 \, &\cause \, #2}}}
	\NewDocumentCommand\cdrule{m+g}{%
		\IfNoValueTF{#2}
			{\ensuremath{#1.}}
			{\ensuremath{#1 & \cause & #2}}}

	\newcommand{\algrule}[4]{
	\hbox{{#1}:}& 
	\quad #2 ~\longrightarrow~ #3 
	\hbox{~ if } #4\\
	}

	\newcommand{\AlgoRule}[4]{
	\ifenv{array}
	{\algrule{#1}{#2}{#3}{#4}}
		{ERROR: using AlgoRule in unsupported environment: \@currenvir}
	}

	\newcommand{\ignore}[1]{}

	\newboolean{nocomments}
	\setboolean{nocomments}{false}

	\newboolean{commentmargin}
	\setboolean{commentmargin}{true}

	\newcommand{\namedcomment}[3]{%
		\ifthenelse{\boolean{nocomments}}%
		{}%IF no comments, write nothing
		{%Otherwise
			\ifthenelse{\boolean{commentmargin}}%
				{ {\color{#3} \marginpar{\color{#3}\sc #2}#1}  }%Name in margin
				{  {\color{#3} {\sc #2}: #1}  }%Name not in margin
		}%
	}
	\newcommand{\mnamedcomment}[3]{\ifthenelse{\boolean{nocomments}}{}{{\marginpar{\tiny \color{#3}{\sc #2}:#1}}}}

	\usepackage{soul}

\usepackage[normalem]{ulem} % wavy underlines
\makeatletter
\font\uwavefont=lasyb10 scaled 700
\def\spelling{\bgroup\markoverwith{\lower3.5\p@\hbox{\uwavefont\textcolor{Red}{\char58}}}\ULon}
\def\grammar{\bgroup\markoverwith{\lower3.5\p@\hbox{\uwavefont\textcolor{LimeGreen}{\char58}}}\ULon}
\def\phrasing{\bgroup\markoverwith{\lower3.5\p@\hbox{\uwavefont\textcolor{RoyalBlue}{\char58}}}\ULon}

\newcommand\remove{\bgroup\markoverwith{\textcolor{red}{\rule[0.5ex]{2pt}{0.4pt}}}\ULon}
\makeatother

\newcommand\etal{et al.\@\xspace}

\newcommand\thanksAFTACK{This work was partially supported by  Fonds Wetenschappelijk Onderzoek -- Vlaanderen (project G0B2221N).}

\usepackage{amsthm}
\usepackage{thmtools}

\declaretheorem[style=plain,	name=Theorem,		numberwithin=section]{thm}
\declaretheorem[style=plain,	name=Theorem,		numberlike=thm]{theorem}

\declaretheorem[style=plain,	name=Proposition,	numberlike=thm]{proposition}

\declaretheorem[style=plain,	name=Lemma,		numberlike=thm]{lemma}
\declaretheorem[style=plain,	name=Lemma,		numbered=no]   {lem*}

\declaretheorem[style=plain,	name=Corollary,		numberlike=thm]{corollary}

\declaretheorem[style=definition,	name=Definition,	numberlike=thm]{definition}

\declaretheorem[style=definition,	qed=$\blacktriangle$,	numberlike=thm]{example}

\declaretheorem[style=definition,	qed=$\blacktriangle$,	numbered=no]{ex*}

\declaretheorem[style=remark,	name=Notation,		numbered=no]{nota*}

\declaretheorem[style=remark,	qed=$\blacktriangle$,	name=Note,		numbered=no]{note*}

\usepackage{etoolbox}

\newcommand\setcitation[2]{%
  \csdef{mycommoncitation#1}{#2}}
\newcommand\getcitation[1]{%
  \csuse{mycommoncitation#1}}

\setcitation{IDP}{WarrenBook/DeCatBBD14}
\setcitation{idp}{WarrenBook/DeCatBBD14}
\setcitation{fodot}{tocl/DeneckerT08}
\setcitation{foid}{tocl/DeneckerT08}
\setcitation{FOID}{tocl/DeneckerT08}
\setcitation{cplogic}{journal/tplp/VennekensDB10}
\setcitation{CPlogic}{journal/tplp/VennekensDB10}
\setcitation{CPLogic}{journal/tplp/VennekensDB10}
\setcitation{CP}{fai/Rossi06}
\setcitation{cp}{fai/Rossi06}
\setcitation{EZCSP}{lpnmr/Balduccini11}
\setcitation{KR}{Baral:2003}
\setcitation{ASPComp2}{lpnmr/DeneckerVBGT09}
\setcitation{ASPComp3}{journals/tplp/CalimeriIR14}
\setcitation{ASPComp4}{conf/lpnmr/AlvianoCCDDIKKOPPRRSSSWX13}
\setcitation{ASPComp5}{journals/ai/CalimeriGMR16}
\setcitation{ASPComp6}{jair/GebserMR17}
\setcitation{ASPComp7}{tplp/GebserMR20}
\setcitation{CPSupport}{ictai/DeCat13}
\setcitation{CPsupport}{ictai/DeCat13}
\setcitation{functionDetection}{iclp/DeCatB13}
\setcitation{FunctionDetection}{iclp/DeCatB13}
\setcitation{fodot2asp}{corr/DeneckerLTV19} %TODO replace by journal publication if one is publishedr
\setcitation{Tarskian}{corr/DeneckerLTV19} %Same as the one above 
\setcitation{TarskianSemanticsASP}{corr/DeneckerLTV19} %Same as the one above 
\setcitation{Inca}{iclp/DrescherW12}
\setcitation{csp2asp}{ijcai/DrescherW11}
\setcitation{DPLLT}{cav/GanzingerHNOT04}
\setcitation{AspInPractice}{synthesis/2012Gebser}
\setcitation{ASPInPractice}{synthesis/2012Gebser}
\setcitation{clasp}{ai/GebserKS12}
\setcitation{oclingo}{kr/GebserGKOSS12}
\setcitation{clingo}{iclp/GebserKKOSW16}
\setcitation{gringo}{lpnmr/GebserST07}
\setcitation{cmodels}{aaai/GiunchigliaLM04}
\setcitation{inputster}{tplp/Jansen13}
\setcitation{DLV}{tocl/LeonePFEGPS06}
\setcitation{LearningPaper}{TPLP/BruynoogheBBDDJLRDV} %TODO replace by published version
\setcitation{clog}{iclp/BogaertsVDV14} %TODO replace by journal publication if one is published
\setcitation{foc}{iclp/BogaertsVDV14}
\setcitation{FOC}{iclp/BogaertsVDV14}
\setcitation{inferenceClog}{ecai/BogaertsVDV14}
\setcitation{examplesClog}{nmr/BogaertsVDV14b} %TODO replace by better publication if possible
\setcitation{AFT}{DeneckerMT00}
\setcitation{KBS}{iclp/DeneckerV08}
\setcitation{KBS-invitedtalk}{jelia/Denecker16}
\setcitation{KBPE}{inap/DePooterWD11}
\setcitation{lazyGrounding}{jair/CatDBS15} 
\setcitation{LazyGrounding}{jair/CatDBS15} 
\setcitation{lazygrounding}{jair/CatDBS15} 
\setcitation{lazygroundingASP}{ijcai/BogaertsW18} 
\setcitation{justifications}{lpnmr/DeneckerBS15} 
\setcitation{justificationsAlpha}{ijcai/BogaertsW18} 
\setcitation{ASP}{marek99stable}
\setcitation{satid}{sat/MarienWDB08}
\setcitation{lazyclausegeneration}{constraints/OhrimenkoSC09}
\setcitation{FP}{ACMCS/Hudak89}
\setcitation{GroundingWithBounds}{jair/WittocxMD10}
\setcitation{GroundWithBounds}{jair/WittocxMD10}
\setcitation{SAT}{faia/SilvaLM09}
\setcitation{HandbookOfSAT}{faia/2009-185}
\setcitation{LTC}{iclp/Bogaerts14}
\setcitation{SPSAT}{ictai/DevriendtBMDD12}
\setcitation{BreakID}{sat/DevriendtBBD16}
\setcitation{breakid}{sat/DevriendtBBD16}
\setcitation{LCG}{stuckeyLCG}
\setcitation{MiniZinc}{conf/cp/NethercoteSBBDT07}
\setcitation{minizinc}{conf/cp/NethercoteSBBDT07}
\setcitation{amadini}{cpaior/AmadiniGM13}
\setcitation{bootstrapping}{ngc/BogaertsJDJBD16}
\setcitation{Bootstrapping}{ngc/BogaertsJDJBD16}
\setcitation{GroundedFixpoints}{ai/BogaertsVD15}
\setcitation{PartialGroundedFixpoints}{ijcai/BogaertsVD15}
\setcitation{LogicBlox}{datalog/GreenAK12}
\setcitation{proB}{journals/sttt/LeuschelB08}
\setcitation{NaturalInductions}{KR/DeneckerV14} %TODO replace by journal publication if one is published
\setcitation{LP}{jacm/EmdenK76}
\setcitation{SMT}{faia/BarrettSST09}
\setcitation{AF}{ai/Dung95}
\setcitation{ADF}{kr/BrewkaW10}
\setcitation{af}{ai/Dung95}
\setcitation{adf}{kr/BrewkaW10}
\setcitation{ADFRevisited}{ijcai/BrewkaSEWW13}
\setcitation{adfrevisited}{ijcai/BrewkaSEWW13}
\setcitation{DefaultLogic}{ai/Reiter80}
\setcitation{DL}{ai/Reiter80}
\setcitation{AEL}{mo85}
\setcitation{minisat}{sat/EenS03}
\setcitation{completion}{adbt/Clark78}
\setcitation{ClarkCompletion}{adbt/Clark78}
\setcitation{wasp}{lpnmr/AlvianoDFLR13}
\setcitation{minisatid}{ictai/DeCat13}
\setcitation{lcg}{stuckeyLCG}
\setcitation{CEGAR}{jacm/ClarkeGJLV03}
\setcitation{cegar}{jacm/ClarkeGJLV03}
\setcitation{CuttingPlane}{or/DantzigFJ54}
\setcitation{kodkod}{tacas/TorlakJ07}
\setcitation{cdcl}{Marques-SilvaS99}
\setcitation{CDCL}{Marques-SilvaS99}
\setcitation{1UIP}{iccad/ZhangMMM01}
\setcitation{relevance}{ijcai/JansenBDJD16}
\setcitation{relevance-implementation}{aspocp/JansenBDJD16}
\setcitation{WFS}{GelderRS91}
\setcitation{wfs}{GelderRS91}
\setcitation{UnfoundedSet}{GelderRS91}
\setcitation{UFS}{GelderRS91}
\setcitation{stablesemantics}{iclp/GelfondL88}
\setcitation{StableSemantics}{iclp/GelfondL88}
\setcitation{shatter}{Shatter}
\setcitation{sbass}{drtiwa11a}
\setcitation{lparsemanual}{url:lparsemanual}
\setcitation{AIC}{ppdp/FlescaGZ04}
\setcitation{templates}{tplp/DassevilleHJD15}
\setcitation{templates2}{iclp/DassevilleHBJD16}
\setcitation{sat-to-sat}{aaai/JanhunenTT16}
\setcitation{sat-to-sat-qbf}{bnp/BogaertsJT16}
\setcitation{sat-to-sat-QBF}{bnp/BogaertsJT16}
\setcitation{sat-to-sat-SO}{kr/BogaertsJT16}
\setcitation{XSB}{SwiW12}
\setcitation{KCmap}{jair/DarwicheM02}
\setcitation{TLA}{DBLP:books/aw/Lamport2002}
\setcitation{EventB}{BookAbrial2010}
\setcitation{MX}{MitchellT05}
\setcitation{MIP}{Sierksma96}
\setcitation{perefectmodel}{minker88/Przymusinski88}
\setcitation{PerfectModel}{minker88/Przymusinski88}
\setcitation{perfectmodel}{minker88/Przymusinski88}
\setcitation{SafeInductions}{ijcai/BogaertsVD17}
\setcitation{AIC}{ppdp/FlescaGZ04}
\setcitation{aic}{ppdp/FlescaGZ04}
\setcitation{alpha}{lpnmr/Weinzierl17}
\setcitation{omiga}{jelia/Dao-TranEFWW12}
\setcitation{gasp}{fuin/PaluDPR09}
\setcitation{asperix}{lpnmr/LefevreN09a}
\setcitation{CTL}{lop/ClarkeE81}
\setcitation{AFT-AIC}{ai/BogaertsC18}
\setcitation{UltimateApproximator}{DeneckerMT04}
\setcitation{KripkeKleene}{Fitting85}
\setcitation{AFT-HO}{corr/CharalambidisRS18} %TODO replace
\setcitation{HereThere}{Heyting30}
\setcitation{dAEL}{ijcai/HertumCBD16}
\setcitation{SDD}{ijcai/Darwiche11}
\setcitation{HEX}{ijcai/EiterIST05}
\setcitation{wADF}{aaai/BrewkaSWW18}
\setcitation{wADFfix}{corr/BrewkaSWW18}
\setcitation{TransitionSystems}{jacm/NieuwenhuisOT06}
\setcitation{galliwasp}{lopstr/MarpleG12}
\setcitation{GalliWasp}{lopstr/MarpleG12}
\setcitation{clingcon}{tplp/BanbaraKOS17}
\setcitation{lp2sat}{birthday/JanhunenN11}
\setcitation{lp2mip}{LIU12}
\setcitation{lp2diff}{lpnmr/JanhunenNS09}
\setcitation{lp2acyc}{ecai/GebserJR14}
\setcitation{PB}{faia/RousselM09}
\setcitation{pb}{faia/RousselM09}
\setcitation{CuttingPlanes}{dam/CookCT87}
\setcitation{RoundingSAT}{ijcai/ElffersN18}
\setcitation{PRS}{aaai/DixonG02}
\setcitation{sat4j}{jsat/BerreP10}
\setcitation{SAT4J}{jsat/BerreP10}
\setcitation{mingo}{LIU12}
\setcitation{pbmodels}{lpnmr/LiuT05}
\setcitation{HEF-LP}{lpnmr/GebserLL07}
\setcitation{HCF-LP}{amai/Ben-EliyahuD94}
\setcitation{aspcore2}{tplp/CalimeriFGIKKLM20}
\setcitation{AspCore2}{tplp/CalimeriFGIKKLM20}
\setcitation{SemanticWeb}{book/AntoniouGHH12}
\setcitation{maxsat}{faia/LiM09}
\setcitation{MaxSAT}{faia/LiM09}
\setcitation{MAXSAT}{faia/LiM09}
\setcitation{QBF}{faia/BuningB09}
\setcitation{qbf}{faia/BuningB09}
\setcitation{SMUS}{cp/IgnatievPLM15}
\setcitation{}{}
\setcitation{}{}
\setcitation{}{}
\setcitation{}{}
\setcitation{}{}
\setcitation{}{}
\setcitation{}{}
\setcitation{}{}
\setcitation{}{}
\setcitation{}{}
\setcitation{}{}
\setcitation{}{}
\setcitation{}{}
\setcitation{}{}
\setcitation{}{}
\setcitation{}{}
\setcitation{}{}
  
\newcommand\refto[1]{%
      \ifcsname mycommoncitation#1\endcsname%
      \getcitation{#1}%
      \else%
      #1%
      \fi%
      }
      
\newcommand\mycite[1]{%
      \ifcsname mycommoncitation#1\endcsname%
   \cite{\getcitation{#1}}%
  \else%
    \cite{#1}%
  \fi%
}	
  
\usepackage[capitalize]{cleveref}

\usepackage{amsmath}
\usepackage{graphicx}
\usepackage{multirow}
\usepackage{mathtools}
\usepackage{amssymb}
\usepackage{amsthm}

\usepackage{mathrsfs}
\usepackage{tikz}
\usetikzlibrary{arrows}
\tikzset{vertex/.style = {shape=circle,draw,minimum size=1em}}
\tikzset{edge/.style = {->,> = latex'}}
\usepackage{tkz-berge}
\usepackage{tikz-cd}
\usetikzlibrary{trees}
\usepackage{mdframed}

\crefname{conjecture}{Conjecture}{Conjectures}

\newcommand\setl[1]{\m{\left\lbrace #1 \right\rbrace}}

\newcommand{\tild}{\m{\mathord{\sim}}}
\newcommand{\suppletter}{\m{\mathcal{S}}}

\newcommand{\suppopsys}[1]{\m{\suppletter_{#1}}}

\newcommand{\interp}{\m{\mathcal{I}}}
\newcommand{\be}{\m{\mathcal{B}}}
\newcommand{\besub}[1]{\m{\be_{\mathrm{#1}}}}
\newcommand{\bekk}{\m{\besub{KK}}}

\newcommand{\besp}{\m{\besub{sp}}}
\newcommand{\F}{\m{\mathcal{F}}}
\newcommand{\jf}{\m{\mathcal{J}\kern-0.2em\mathcal{F}}}
\newcommand{\jfcomplete}[1][\rules]{\m{\left\langle \F, \Fd, #1\right\rangle}}
\newcommand{\js}{\m{\mathcal{J}\kern-0.2em\mathcal{S}}}
\newcommand{\jscomplete}[1][\be]{\m{\left\langle \F, \Fd, \rules, #1\right\rangle}}
\newcommand{\lf}{\m{\mathcal{L}}}

\newcommand{\Fd}{\m{\F_d}}
\newcommand{\Fo}{\m{\F_o}}
\newcommand{\pathstyle}[1]{\mathbf{#1}}
\newcommand{\branch}{\m{\pathstyle{b}}}
\newcommand{\branches}{\m{B}}

\newcommand{\justifications}{\m{\mathfrak{J}}}
\newcommand{\edges}{\m{E}}

\newcommand{\sel}{\mathcal{S}}

\DeclareMathOperator{\C}{C}
\DeclareMathOperator{\CC}{CC}

\DeclareMathOperator{\suppvalue}{SV}
\DeclareMathOperator{\jval}{val}

\DeclareMathOperator{\im}{Im}

\newcommand\restrict[2]{{%
    \left.\kern-\nulldelimiterspace%
    #1%
    \vphantom{|}%
    \right|_{#2}%
}}

\newenvironment{research}[1]
{\mdfsetup{
    frametitle={\colorbox{white}{\space#1\space}},
    innertopmargin=0.5em,
    frametitleaboveskip=-\ht\strutbox,
    frametitlealignment=\center
  }
  \begin{mdframed}%
  }
  {\end{mdframed}}

\newcommand{\jframe}[1]{\left\{\begin{array}{l}#1\end{array}\right\}}

\renewcommand{\F}{\m{\mathbb{F}}}
\renewcommand{\jf}{\m{\mathbb{JF}}}
\renewcommand{\js}{\m{\mathbb{JS}}}
\renewcommand{\sel}{\m{s}}
\renewcommand{\interp}{\m{I}}
\renewcommand{\suppletter}{\m{S}}

\title{Tree-Like Justification Systems are Consistent\thanks{\thanksAFTACK}}
\author{Simon Marynissen
\institute{KU Leuven, Leuven, Belgium}
\institute{Vrije Universiteit Brussel, Brussel, Belgium}
\email{simon.marynissen@kuleuven.be}
\and
Bart Bogaerts
\institute{Vrije Universiteit Brussel, Brussel, Belgium}
\email{bart.bogaerts@vub.be}
}
\def\titlerunning{Tree-Like Justification Systems are Consistent}
\def\authorrunning{S. Marynissen \& B. Bogaerts}
\begin{document}
\maketitle

\begin{abstract}
Justification theory is an abstract unifying formalism that captures semantics of various non-monotonic logics.
One intriguing problem that has received significant attention is the \emph{consistency problem}: under which conditions are justifications for a fact and justifications for its negation suitably related.
Two variants of justification theory exist: one in which justifications are trees and one in which they are graphs.
In this work we resolve the consistency problem once and for all for the tree-like setting by showing that all reasonable tree-like justification systems are consistent.
\end{abstract}

\section{Introduction}

% \todo{The main weakness of the paper is lack of a convincing discussion of relevance of the presented results to logic programming. While the authors cite papers where such connections have been shown, a more in-depth discussion would make the paper stronger. Further, I felt the implications of the main result for logic programing should have been discussed. If this can be done in the final version of the paper, it will make the paper much better.
% }
%
% \todo{more intuitions for the main definitions would help}

Justification theory \cite{lpnmr/DeneckerBS15} is a unifying theory to capture semantics of non-monotonic logics.
Largely thanks to its abstract nature, it is a powerful framework with many use cases.
First, it provides a mechanism to \emph{define new logics} based on well-known principles in a uniform way, as well as to \emph{transfer results} between domains.
Second, it brings order in the zoo of logics and semantics, by enabling a \emph{systematic comparison} between multiple semantics for a single logic and between different logics, for instance by answering the question whether a certain  semantics of a given logic coincides with a semantics of another logic.
Third, building on the notion of \emph{nested justification systems}\footnote{Nested justification systems were originally defined by Denecker \etal \cite{lpnmr/DeneckerBS15}, but have remained largely unexplored since then. In a companion paper to this paper \cite{tplp/MarynissenBDH22}, we provide a systematic study of nested systems and their properties}, it facilitates \emph{modular definitions of knowledge representation languages and semantics}.

Justification theory builds on the semantic notion of a \emph{justification}, which can intuitively be understood as an explanation (in the form of a tree or a directed graph) as to why a certain fact is true or false.
For this reason, logics of which the semantics is captured by justification theory automatically come with a mechanism of \emph{explanation}, which is of increasing importance for societal and legal reasons.
On top of that, justifications have repeatedly proven to be useful algorithmically.
They have been used in the unfounded set algorithm of Gebser \etal \cite{iclp/GebserKKS09}, for improving lazy grounding algorithms \cite{ijcai/BogaertsW18}, as well as to speed-up parity game solvers \cite{vmcai/LapauwBD20}.

The roots of justification theory can be traced back to the doctoral thesis of Denecker \cite{DeneckerPhD93}, where it was developed as a framework for studying semantics of logic programs.
Later, Denecker \etal \cite{lpnmr/DeneckerBS15} developed a more general theory, aiming to also capture other knowledge representation formalisms, such as abstract argumentation \mycite{AF}, and nested least and greatest fixpoint definitions \cite{tplp/HouDD10}.
A notable difference between the original work of Denecker \cite{DeneckerPhD93} and the theory of Denecker \etal \cite{lpnmr/DeneckerBS15} is that in the former work, a justification is a tree where the nodes are labeled with literals, while in the latter a justification is a (directed) graph.
The relationship between these two formalisms was studied among others by Marynissen \etal \cite{tplp/MarynissenBD20}.
For clarity, we will refer to \emph{tree-like} and \emph{graph-like} justifications when the distinction is important.

In justification theory, whether or not a justification is ``good'' is determined by a \emph{branch evaluation}.
A branch evaluation is a function that associates to each path through a justification (i.e., to each branch of the justification) a value.
However, a priori,  it is not clear that each branch evaluation induces a well-defined semantics.
For this to be the case, intuitively the justifications of a fact $x$ and those of its negation $\tild x$ should be suitably related.
Indeed, we cannot accept that there is both an explanation that $x$ is true and an explanation that $\tild x$ is true.
In general, to get a well-defined semantics we will need that the best possible justification for $x$ and the best possible justification for $\tild x$ are complementary (either one of them is true and the other false, or both are unknown).
% should have a ``good'' justification if and only its negation $\tild x$ has no ``good'' justification.
The problem of determining whether or not a branch evaluation induces a well-defined semantics is known as the \emph{consistency problem} and has been studied in several papers.
\begin{itemize}
  \item Denecker \cite[Theorem 4.3.1]{DeneckerPhD93} studied the consistency problem for \emph{tree-like} justifications for three specific branch evaluations (corresponding to completion semantics, stable semantics and well-founded semantics in logic programming).
%   That proof of consistency of three branch evaluations spans six pages.
  \item Marynissen \etal \cite{nmr/MarynissenPBD18} were the first to exhibit a branch evaluation that (for graph-like justifications) is \emph{not} consistent. Moreover, they showed that for \emph{graph-like} justifications, four branch evaluations (in addition to the three mentioned above, also the Kripke-Kleene branch evaluation) are guaranteed to be consistent.
  \item Marynissen \etal \cite{tplp/MarynissenBD20} investigated the relationship between justification theory and games over graphs  \cite{concur/GimbertZ05}.
  They used this relationship to identify some key properties that, when satisfied by a branch evaluation, guarantee that that branch evaluation is consistent for graph-like \emph{and} tree-like justifications, but only in the context of a \emph{finite} fact space.
\end{itemize}
Despite all these efforts, there is no clear understanding yet of what it is that makes a branch evaluation consistent.
Moreover, proofs of consistency of individual branch evaluations often span several pages.
In this paper, we resolve this question once and for all for tree-like justifications.
Our main theorem states that: \emph{for tree-like justifications, every reasonable branch evaluation is consistent}.
The proof is surprisingly simple (compared to earlier proofs for individual branch evaluations) and is completely included in the main text.
The results presented in this paper are part of the PhD thesis of the first author \cite{phd/Marynissen22}.

The rest of this paper is structured as follows.
We recall some basic definitions of justification theory (focusing solely on tree-like justifications) in  \cref{sec:prelims} and state the consistency problem in \cref{sec:consistency}.
In \cref{sec:core} we present the core theoretic results (in somewhat more generality) that are subsequently used to prove our main result (namely that every reasonable branch evaluation is consistent)  in \cref{sec:result}.
We conclude in \cref{sec:concl}.

\section{Preliminaries}\label{sec:prelims}

We present the core definitions of justification theory, based on the formalization of Marynissen \etal \cite{tplp/MarynissenBD20,ijcai/MarynissenBD21}.

In the rest of this paper, let $\F$ be a set, referred to as a \emph{fact space},
such that $\lf = \setl{\Tr, \Fa, \Un} \subseteq \F$,
where $\Tr$, $\Fa$ and $\Un$ have the respective meaning \emph{true}, \emph{false}, and \emph{unknown}.
The elements of $\F$ are called \emph{facts}.
The set $\lf$ behaves as the three-valued logic with truth order $\leqt$ given by  $\Fa \leqt \Un \leqt \Tr$.
We assume that $\F$ is equipped with an involution $\tild: \F \rightarrow \F$ (i.e., a bijection that is its own inverse)
such that $\tild \Tr=\Fa$, $\tild \Un=\Un$, and $\tild x \neq x$ for all $x \neq \Un$.
For any fact $x$, $\tild x$ is called the \emph{complement} of $x$.
An example of a fact space is the set of literals over a propositional vocabulary $\Sigma$ extended with $\lf$ where $\tild$ maps a literal to its negation.
For any set $A$ we define $\tild A$ to be the set of elements of the form $\tild a$ for $a \in A$.
We distinguish two types of facts: \emph{defined} and \emph{open} facts.
The former are accompanied by a set of rules that determine their truth value.
The truth value of the latter is not governed by the rule system but comes from an external source or is fixed (as is the case for logical facts).

\begin{definition}
  A \emph{justification frame} $\jf$ is a tuple $\jfcomplete$ such that
  \begin{itemize}
    \item $\Fd$ is a  subset of $\F$ closed under $\tild$, i.e.\ $\tild \Fd = \Fd$; facts in $\Fd$ are called \emph{defined};
    \item no logical fact is defined: $\lf \cap \Fd = \emptyset$;
    \item $\rules \subseteq \Fd \times 2^\F$;
    \item for each $x \in \Fd$, $(x, \emptyset) \notin R$ and there is an element $(x, A) \in R$ for $\emptyset \neq A \subseteq \F$.
  \end{itemize}
\end{definition}
The set of \emph{open} facts is denoted as $\Fo:=\F\setminus\Fd$.
An element $(x, A) \in \rules$ is called a \emph{rule} with \emph{head} $x$ and \emph{body} (or \emph{case}) $A$.
The set of cases of $x$ in $\jf$ is denoted as $\jf(x)$.
Rules $(x, A) \in \rules$ are often denoted as $x \gets A$ and if $A=\setl{y_1, \ldots, y_n}$, we often write $x \gets y_1, \ldots, y_n$.
% Cases are allowed to be infinite.

Logic programming rules can easily be transformed into rules in a justification frame.
However, in logic programming, only rules for positive facts are given; never for negative facts.\footnote{In some extensions of logic programming ``classical negation'' is allowed in the head of rules. In that setting, an expression $\lnot p$ is a shorthand for an arbitrary atom that can never be true together with $p$. This is significantly different from our setting where $\tild p$ simply means that $p$ is false, and hence a rule with $\tild p$ in the head states a condition under which $p$ is false. To further   illustrate this difference: if there are no rules for $\tild p$, this means in our setting that $\tild p$ cannot be true (hence $p$ cannot be false). In the aforementioned extensions of logic programming, a lack of rules for $\lnot p$ entails nothing about $p$. }
Hence, in order to apply justification theory to logic programming, a mechanism for deriving rules for negative literals is needed as well.
Similarly, in the setting of argumentation we can naturally derive rules for negative facts (given attack relations), but rules for positive facts are less obvious.
For this, a technique called \emph{complementation} was invented \cite{lpnmr/DeneckerBS15}; it is a generic mechanism that allows turning a set of rules for $x$ into a set of rules for $\tild x$.

Complementation makes use of so-called selection functions for $x$.
A \emph{selection function} for $x$ is a mapping $\sel \colon \jf(x) \rightarrow \F$ such that $\sel(A) \in A$ for all rules of the form $x \gets A$.
Intuitively, a selection function chooses an element from the body of each rule of $x$.
% The existence of selection functions in general is guaranteed by the axiom of choice.
For a selection function $\sel$, the set $\{\sel(A) \mid A \in \jf(x)\}$ is denoted by $\im(\sel)$.
Selection functions can be used to construct rules for $\tild x$ from rules for $x$: intuitively, a fact $\tild x$ can be derived if every rule for $x$ fails.
Formally, if $\rules$ is a set of rules, the complement $\C(\rules)$ is the set of elements of the form  $\tild x \gets \tild \im(\sel)$ with $x$ defined in $\rules$ and $\sel$ a selection function for $x$ in $\rules$.
The \emph{complementation} of a set of rules $\rules$ is $\CC(\rules):=\rules \cup \C(\rules)$.

In general, we will be interested in justification frames where the rules for $x$ and the rules for $\tild x$ are suitably related.
It does not necessarily have to be the case that our frame is obtained by complementation, but still we want to ensure that the rules are compatible.
This has been studied intensively by Marynissen \etal \cite{nmr/MarynissenPBD18}, who have given several equivalent characterizations of when a justification frame is complementary. Here, we give only one such characterization:
\begin{definition}\label{def:charcomplementarywithselection}
  Let $\jf=\jfcomplete$ be a justification frame.
  We call $\jf$ \emph{complementary} if for every $x \in \Fd$ the following hold:
  \begin{enumerate}
    \item\label{def:charcomplementarywithselection:item1} for every selection function $\sel$ for $x$ in $\rules$, there exists an $A \in \jf(\tild x)$ with $A \subseteq \tild \im(\sel)$;
    \item\label{def:charcomplementarywithselection:item2} for every $A \in \jf(x)$, there exists a selection function $\sel$ for $\tild x$ in $\rules$ with $\tild \im(\sel) \subseteq A$.
  \end{enumerate}
\end{definition}
The first item states that if we find a way to block every possible rule for $x$ (by taking the complement of all the elements selected by a selection function $\sel$ for $x$),
then there should be a rule that derives $\tild x$, i.e., there should be some rule $\tild x \leftarrow A$ with $A$ containing only facts in $\tild \im(\sel)$.
The second item states the other direction, namely that whenever we can derive $x$ (by means of a rule $x\lrule A$), it cannot be possible to $\tild x$, i.e., at least one fact in each rule for $\tild x$ should be blocked. This is expressed again by means of a selection function: there should exist a selection function that selects only facts that are the complement of a fact in $A$.
For more intuition regarding complementarity, we refer the reader to Marynissen \etal \cite{nmr/MarynissenPBD18,tplp/MarynissenBD20}.

\begin{example}
    The justification frame
  \begin{align*} p&\lrule q, \tild r\\
    \tild p&\lrule \tild q\\
%     \tild p&\lrule r
  \end{align*}
  is not complementary. Intuitively, the first rule states that $p$ holds whenever $q$ is true and $r$ is false; this is the only case in which $p$ can be derived.
  Since this is the only rule for $p$, we expect $p$ to be false (i.e., $\tild p$ to be true) whenever $q$ is false or $r$ is true; the first case is present, but the second case for $\tild p$ is missing.
%   a rule for $\tild p$ indicating that whenever $r$ is true, $p$ cannot be true is missing.
After adding the rule
  \begin{equation}\tild p \lrule r\label{eq:firstr}\end{equation}
  the frame becomes complementary.
  If we add a further rule
  \begin{equation}\tild p \lrule r,q,\label{eq:seqr}\end{equation}
  the frame is still complementary; intuitively, \eqref{eq:seqr} is a redundant rule, in the sense that it is weaker than \eqref{eq:firstr}.
\end{example}

\begin{definition}
%   \simon{Might be left out, if found obvious}
  A \emph{directed labeled graph} is a quadruple $(N, L, \edges, \ell)$ where $N$ is a set of nodes, $L$ is a set of labels, $\edges \subseteq N \times N$ is the set of edges, and $\ell\colon N \rightarrow L$ is a function called the labeling.
  An \emph{internal} node is a node with outgoing edges and a \emph{leaf} node is one without outgoing edges.
\end{definition}

% We now define justifications, which are called tree-like justifications in \cite{tplp/MarynissenBD20}.
\begin{definition}
  Let $\jf=\jfcomplete$ be a justification frame.
  A \emph{(tree-like) justification} $J$ in $\jf$ is a directed labeled graph $(N, \Fd, \edges, \ell)$ such that
  \begin{itemize}
    \item the underlying undirected graph is a forest,~i.e., is acyclic;
    \item for every internal node $n \in N$ it holds that $\ell(n) \gets \{\ell(m) \mid (n,m) \in \edges\} \in \rules$.
  \end{itemize}
\end{definition}

\begin{definition}
  A justification is \emph{locally complete} if it has no leaves with label in $\Fd$.
  We call $x \in \Fd$ a \emph{root} of a justification $J$ if there is a node $n$ labeled $x$ such that every node is reachable from $n$ in $J$.
\end{definition}

We write $\justifications(x)$ for the set of locally complete justifications rooted in a node labeled $x$.

\begin{definition}
  Let $\jf$ be a justification frame.
  A $\jf$-\emph{branch} is either an infinite sequence in $\Fd$ or a finite
%   non-empty  TODO I REMOVED THIS. TO CHECK IF WE WANT THIS
sequence in $\Fd$ followed by an element in $\Fo$.
  For a justification $J$ in $\jf$, a $J$-\emph{branch} starting in $x \in \Fd$ is a path in $J$ starting in $x$ that is either infinite or ends in a leaf of $J$.
  We write $\branches_J(x)$ to denote the set of $J$-branches starting in $x$.
\end{definition}
Not all $J$-branches are $\jf$-branches  since they can end in nodes with a defined fact as label.
However, if $J$ is locally complete, any $J$-branch is also a $\jf$-branch.

We denote a branch $\branch$ as $\branch:x_0 \rightarrow x_1 \rightarrow \cdots$ and define $\tild \branch$ as $\tild x_0 \rightarrow \tild x_1 \rightarrow \cdots$.

\begin{definition}
  A \emph{branch evaluation} $\be$ is a mapping that maps any $\jf$-branch to an element in $\F$ for all justification frames $\jf$.  A branch evaluation $\be$ \emph{respects negation} if $\be(\tild \branch) = \tild \be(\branch)$ for any branch $\branch$.
  A justification frame $\jf$ together with a branch evaluation $\be$
  forms  % SINGULAR: https://www.grammar.com/coupled-with-as-well-as-along-with-together-with-not-to-mention
a \emph{justification system} $\js$, which is presented as a quadruple $\jscomplete$.
  % \begin{definition}
  % \end{definition}
\end{definition}
We now define some branch evaluations that induce semantics corresponding to the equally named semantics of logic programs.
% We start with some particular branch evaluations.

\begin{definition}
  The \emph{supported} (completion) branch evaluation $\besp$ maps $x_0 \rightarrow x_1 \rightarrow \cdots$ to $x_1$ and branches that consist only of a singe open fact $x_0$ to $x_0$.
  The \emph{Kripke-Kleene} branch evaluation $\bekk$ maps finite branches to their last element and infinite branches to $\Un$.
\end{definition}

Other branch evaluations have been defined as well, for instance stable and well-founded branch evaluations, which correspond to the equally-named semantics of logic programming \cite{iclp/GelfondL88,\refto{WFS}}.
We refer the reader to the original work introducing justification theory \cite{lpnmr/DeneckerBS15} for their definitions.

\begin{definition}
  A \emph{(three-valued) interpretation} of $\F$ is a function $\interp:\F \rightarrow \lf$ such that $\interp(\tild x) = \tild \interp(x)$ and $\interp(\ell) = \ell$ for all $\ell \in \lf$.
%   The set of interpretations of $\F$ is denoted by $\mathfrak{I}_\F$.
\end{definition}

\begin{definition}
  Let $\js=\jscomplete$ be a justification system, $\interp$ an interpretation of $\F$, and $J$ a locally complete justification in $\js$.
  Let $x \in \Fd$ be a label of a node in $J$.
  The \emph{value} of $x \in \Fd$ by $J$ under $\interp$ is defined as $\jval(J, x, \interp) = \min_{\branch \in \branches_J(x)} \interp(\be(\branch))$,
  where $\min$ is the minimum with respect to $\leqt$.

  The \emph{supported value} of $x \in \F$ in $\js$ under $\interp$ is defined as
  \begin{align*}
  \suppvalue_{\js}(x, \interp) &= \max_{J \in \justifications(x)} \jval(J,x,\interp)\text{  for $x \in \Fd$  }\\
  \suppvalue_{\js}(x, \interp) &=\interp(x) \text{ for $x \in \Fo$}.
  \end{align*}
  %   if $x \in \Fd$ and $\suppvalue_t(x, \interp)=\interp(x)$ if $x \in \Fo$.
\end{definition}

In other words, the value of a fact in a justification is the value of the worst branch starting in that fact, and the supported value of a fact in an interpretation is the value of that fact in its best justification.
Models are then defined as those interpretation in which the supported value of each fact equals their actual value.
Every branch evaluation induces a different class of ``models''. For instance, models under the supported branch evaluation correspond to supported models in logic programming.

When $\js$ is clear from the context, we will drop the subscript and just write $\suppvalue(x,\interp)$ for $\suppvalue_{\js}(x,\interp)$.

\begin{definition}
  Let $\js=\jscomplete$ be a justification system.
  An $\F$-interpretation $\interp$ is a $\js$-model if for all $x \in \Fd$, $\suppvalue(x, \interp) = \interp(x)$.
  If $\js$ consists of $\jf$ and $\be$, then a $\js$-model will also be called a $\be$-model of $\jf$.
\end{definition}

\section{The Consistency Problem}\label{sec:consistency}
We defined \emph{models} of a justification system by a kind of a fixpoint equation: for $\interp$ to be a model, it must be a fixpoint of the operator $\suppopsys{\js}$ that maps $\interp$ to the interpretation $\suppopsys{\js}(\interp)$, which is the function
\[ \F\to \lf: x\mapsto \suppvalue(x,\interp).\]
% \simon{Hier overload je de SV notatie, in alle vorige paper en PhD, gebruiken we $\suppopsys{\js}$, alsook, zoek je fixpoints van $\suppvalue(\cdot)$, ipv wat hier boven staat}
However, one intriguing problem is that domain and range of the support operator are not equal: the range of the support operator is the set of functions from $\F$ to $\lf$, while the domain is the set of $\F$-interpretations, which are functions from $\F$ to $\lf$ with some additional properties such as $\interp(\tild x) = \tild \interp(x)$ for all $x \in \F$.
We might then wonder under which conditions, these properties will be guaranteed for $\suppopsys{\js}(\interp)$.

On top of that, there is a more fundamental reason why this property is important.
Explanations are of growing importance in various subdomains of artificial intelligence.
In our setting, a justification with value $\Tr$ for $x$ serves as an explanation of why $x$ is true.
But what if $x$ is false? Which semantic structure can explain that?
From the definition of supported value, it can be seen that $x$ is false if \emph{there are no justifications} with a better value (than $\Fa$) for $x$.
But, the absence of such justifications is difficult to argue.
The question that then remains is: how to show that there are no better justifications for $x$.
The most obvious solution is considering a justification of $\tild x$.
Indeed, intuitively, an explanation why the negation of $x$ is true should explain why $x$ is false.
However, this method implicitly assumes that $\suppvalue(\tild x, \interp) = \tild \suppvalue(x,\interp)$, thus motivating the following definition.

% This is a fundamental property, otherwise justification semantics is unsound: we have a justification explaining $x$ is $\ell$, while there is no justification that explains that $\tild x$ has value $\tild \ell$.
\begin{definition}
  A justification system $\js$ is \emph{consistent} if $\suppvalue(\tild x, \interp) = \tild \suppvalue(x,\interp)$ for every $x \in \Fd$ and every $\F$-interpretation $\interp$.
\end{definition}

We are now ready to state the consistency problem, which is the central research question of this article.
% The first and most important research question of this thesis is the following.

\begin{research}{Consistency problem}
  When is a justification system consistent?
  In particular, what properties do branch evaluations and justification frames need to have to ensure that the justification system is consistent?
\end{research}

Consistency is a reasonable assumption that, unfortunately, is not always satisfied.
An obvious way to not satisfy it is by having unrelated rules for $x$ and $\tild x$.

\begin{equation*}
  \jframe{
    x \gets \Tr \\
    \tild x \gets \Tr
  }
\end{equation*}
Of course, in this example we cannot expect that the justifications for $x$ and $\tild x$ are related because their rules are contradictory.
While justification theory in principle works with such rule sets, most of the theory has focused on complementary justification frames, where the rules for $x$ and $\tild x $ are suitably related.
Another way in which consistency can be violated is if $\be$ does not respect negation. % \simon{Al gedefinieerd?} Ja, def 2.8
Again, in its most general form, justification theory allows such strange branch evaluations (e.g., one that maps every branch to $\Tr$, in which every justification is always true), however there are no applications of such branch evaluations.

In fact, the main result of our paper is that for \emph{tree-like} justifications, the two aforementioned syntactic properties are enough to guarantee consistency:

\begin{theorem}[Main theorem]\label{thm:main}
  Let $\js=\jscomplete$ be a justification system. If $\js$ is complementary and $\be$ respects negation, then $\js$ is (tree-like) consistent.
\end{theorem}
We do want to stress the fact that this only holds for tree-like justification systems: Marynissen \etal \cite{nmr/MarynissenPBD18} gave an example of a branch evaluation that respects negation, on a complementary justification frame,  that is \emph{not} consistent for graph-like justifications (which we do not consider in this paper).

\newcommand\branchsel{\m{\mathbb{B}}}
\newcommand\glue[3]{\m{#1\stackrel{#2}{\rightarrow}#3}}
\newcommand\justificationsRootedIn[1]{\m{\justifications(#1)}}

\section{Constructing a Justification}\label{sec:core}
In this section, we will provide the core theoretic results that are needed to prove our main theorem.
In short, what we aim to show here is that if there are no ``good'' justifications for $x$, then we can construct a ``good'' justification for $\tild x$.
We will prove this in some more generality, making use of a new concept: a \emph{branch selection}: a set of branches starting in a certain fact $x$ that contains at least one branch from every justification of $x$.
The main result of this section (Theorem \ref{thm:core}) then states that for every branch selection $\branchsel$ for $x$, we can construct a justification for $\tild x$ that only has branches in $\tild \branchsel$.

Throughout this section, we fix a justification system $\js=\jscomplete$ with $\jf=\jfcomplete$.

\begin{definition}
%   Let $\jf$ be a justification frame.
%
  A \emph{branch selection} for $x$ in \jf is a set $\branchsel$ of branches starting in $x$ such that
  for each locally complete justification $J$ rooted in $x$, \branchsel contains at least one $J$-branch.
  %(starting in the root).
%   \simon{Ik vind deze verwoording een beetje raar/vaag}
\end{definition}

In what follows, if $x\lrule A$ is a rule, and for each $y\in A$, $J_y$ is a justification\footnote{In case $y$ is an open fact, there is only one such justification, namely a justification with a single node and without edges.} rooted in $y$, we will write $\glue{x}{A}{(J_y)_{y\in A}}$ for the justification in which $A$ are the children of $x$ and such that the subtree rooted in $y\in  A$ equals $J_y$.

\begin{lemma}\label{lem:chooseystar}
Let \branchsel be a branch selection for $x$, and let $x\gets A$ be a rule in $\jf$.
There exists a $y^* \in A$ such that for every justification $J_{y^*} \in \justificationsRootedIn{y^*}$ rooted in $y^*$,
\branchsel contains at least one branch of the form $x\to y^* \to \branch$ with $y^*\to \branch$ a branch starting from the root of $J_{y^*}$.\footnote{In case $y^*$ is an open fact, this thus means that $\branch$ is the empty branch and $y^*\to \branch$ is a branch with only a single element $y^*$.}
% \simon{Is er een reden waarom je een set van branches neemt ipv een branch associeert met elke justification} Ja. Zie e-mail
% \simon{Ik zou het noteren als: $\branchsel$ contains at least one branch of the form $x \to \branch$ with $\branch$ a branch starting from the root $J_{y^*}$,  want dan klopt het ook als $y^* \in \Fo$.}
%
% there exist justifications $(J_y)_{y \in A \setminus \setl{y^*}}$ rooted in the other elements\footnote{Note that if $y\in \Fo$, then $\justificationsRootedIn{y}$ only contains a single justification with single node.} of $A$ such that \branchsel contains at least one branch of $x\to A \to (J_y)_{y\in A}$ starting with $x\to y^*$.
\end{lemma}
\begin{proof}
  Assume by contradiction that such a $y^*$ does not exist.
  This would mean that for each $y \in A$, there exists a $J_y \in \justificationsRootedIn{y}$ such that $\branchsel$ does \emph{not} contain any branches of the form $x\to y \to \branch$ with $y\to\branch$ a branch in $J_y$.
%   \simon{zelfde opmerking als hierboven} FOOTNOTES CLARIFY
  However, if we have such a justification for each $y\in A$, consider the justification
  $\glue{x}{A}{(J_y)_{y\in A}}$. Since this is a locally complete justification of $x$, $\branchsel$ must contain at least one of its branches, which yields a contradiction.
\end{proof}

\begin{corollary}\label{cor:chooseystarbranchsel}
  Let \branchsel be a branch selection for $x$, and let $x\gets A$ be a rule in $\jf$.
  There exists a $y^* \in A$ such that $\branchsel_{y^*} := \{y^*\to\branch \mid x\to y^*\to \branch \in \branchsel\}$ is a branch selection for $y^*$.
\end{corollary}
\begin{proof}
  The only thing we need to show is that $\branchsel_{y^*}$ contains at least one branch from each justification rooted in $y^*$. This follows directly from \cref{lem:chooseystar}.
%   \simon{Op een of andere manier verwijst deze ref naar een andere lemma, maar dat zal wel iets op mijn computer zijn, dus negeer het maar} LIGT AAN DUBBELE REFS. WORDT GEFIXED ALS OUDE TEKST WEG IS.
\end{proof}

\begin{corollary}\label{cor:inductionstep}
  Assume $\js$ %=\jscomplete$ be
  is a complementary justification system.
  If $\branchsel$ is a branch selection for $x$ in $\jf$, then there exists a rule
  $\tild x \gets B_{\tild x}$ for $\tild x$ such that for each $y \in \tild B_{\tild x}$, we have that $\branchsel_{y} := \{y\to\branch \mid x\to y\to \branch \in \branchsel\}$
  %   \simon{zelfde opmerking als hierboven} FOOTNOTES CLARIFY
  is a branch selection for $y$.
%   \begin{itemize}
%     \item There is a rule $x \gets A_y$ with $y \in A_y$.
%     \item For all $J_y \in \justificationsRootedIn{y}$, we can fix justifications $J_z^y \in \justificationsRootedIn{z}$ for each $z \in A_y \setminus \setl{y}$ so that $\selt$ chooses a branch starting with $x \rightarrow y$ in the justification $(x \rightarrow A_y) \oplus J_y \oplus (J_z^y)_{z \in A_y \setminus \setl{y}}$.
%     \item For each $J_y \in \justificationsRootedIn{y}$, let $\selt_y(J_y)$ be the branch satisfying $x \rightarrow \selt_y(J_y) = \selt((x \rightarrow A_y) \oplus J_y \oplus (J_z^y)_{z \in A_y \setminus \setl{y}})$.
%     Then $\selt_y$ is a branch selection function for $y$.
%   \end{itemize}
\end{corollary}
\begin{proof}
  \cref{cor:chooseystarbranchsel} determines a selection function for $x$ in $\jf$, by choosing from each $x \gets A$ a $y^* \in A$.
  Then by complementarity of $\js$, there is a rule $\tild x \gets B_{\tild x}$ such that every element of $B_{\tild x}$ is one of the selected $y^*$.
%   The second point follows from the condition on $y^*$.
%   The third point follows from the second point.
\end{proof}

\begin{theorem}\label{thm:core}
  If $\js$ is a complementary justification system and %and $\be$ respects negation.
  $\branchsel$ is a branch selection for $x$ in $\jf$, then there exists a justification $J'$ rooted in $\tild x$  such that for each branch $\branch'$ of $J'$ starting in the root, $\tild\branch'\in \branchsel$.
\end{theorem}

\begin{proof}
\newcommand\mytree{\m{T}}
\newcommand\mynode{\m{\eta}}
\newcommand\chosenbody[1]{\m{\mathit{chosenBody}(#1)}}
We will first construct a labeled tree $\mytree$ rooted in $\tild x$ with the properties that
\begin{enumerate}
  \item the label of each internal node is a defined fact;
  \item the label of each leaf is an open fact;
  \item for each path in $\mytree$ starting in the root, the corresponing branch (obtained by taking the labels of each node) is the negation of a branch in $\branchsel$.
\end{enumerate}
However, our tree $\mytree$ will not necessarily be a justification: the set of children's labels of an internal node will not necessarily be the body of a rule for the label of that node.
In the second part of the proof, we will then prove that a subtree of $\mytree$ indeed forms a justification.

Our tree $T$ is constructed as follows.
\begin{itemize}
  \item The \textbf{nodes} of $T$ are all finite prefixes of the branches in $\tild \branchsel$. That is, the nodes of $T$ are all sequences $\tild x\to \tild x_1\to \tild x_2\to \dots \to \tild x_n$ such that there is a branch in $\branchsel$ that starts with $x\to x_1\to\dots \to x_n$.
  Each such node is \textbf{labeled} $\tild x_n$.
  \item For each prefix $\tild x\to \tild x_1\to \tild x_2\to \dots \to \tild x_{n-1}\to  \tild x_n$, there is an \textbf{edge} from $\tild x\to \tild x_1\to \tild x_2\to \cdots \to  \tild x_{n-1}$ to $\tild x\to \tild x_1\to \tild x_2\to \dots \to \tild x_n$.
\end{itemize}
By construction, there is a one-to-one correspondence between paths starting in the root of \mytree and branches in \branchsel.
Since all branches in $\branchsel$ are part of a locally complete justification, it is clear that indeed, the labels of internal nodes of this tree are defined facts and the labels of leaves are open facts.

We will now show that we can choose %\simon{construct?} I prefer to emphasize this is a choice -BB
a subtree $J'$ of $\mytree$ that is a locally complete justification.
For each node $\mynode  = \tild x \to \tild x_1 \to \cdots \to  \tild x_n$, we define $\branchsel_\mynode$ as the set of branches $\branch$ starting in $x_n$ such that $x\to x_1\to \dots \to x_{n-1}\to \branch$
% \simon{I would note this as $\cdots x_{n-1} \rightarrow \branch$}
is a branch in $\branchsel$.
For some nodes, $\branchsel_\eta$ is a branch selection function for $x_n$, but this is not guaranteed to be the case for every node.
For \emph{those nodes} $\eta$ (with label $\tild z$) for which $\eta$ \emph{is} a branch selection for $z$,  \cref{cor:inductionstep} guarantees there is a rule $\tild z\lrule B_{\tild z}$ such that for each $\tild y\in B_{\tild z}$,  $\{y\to \branch\mid z\to y \to \branch\in \branchsel_\eta\}$ is also a branch selection function for $y$.
For each such $\eta$, we choose such a rule and for the rest of this proof denote $B_{\tild z}$ as $\chosenbody{\eta}$.
Given this choice, we construct the justification $J$ inductively as follows:
\begin{itemize}
  \item The root $\tild x$ is a node in $J$.
  \item For each node  $\mynode = \tild x \to \tild x_1 \to \tild x_n$ in $J$,
  and for each $\tild y\in \chosenbody{\mynode}$, the node $\tild x \to \tild x_1 \to \cdots \to  \tild x_n\to \tild y$ of $T$ is a node of $J$.
\end{itemize}
Of course, in order for this construction to work, we need to show that for each node in $J$, $\chosenbody{\mynode}$ is well-defined, i.e., that for each node $\eta$ in $J$ with label $\tild z$, $\branchsel_\eta$ is a branch selection function for $z$.
We prove this inductively as well
\begin{itemize}
  \item This claim clearly holds for the root node since the corresponding branch selection function simply equals \branchsel.
  \item If $\mynode = \tild x \to \tild x_1 \to \cdots \to  \tild x_n$ is a node for which $\branchsel_\eta$ is a branch selection function for $x_n\in \Fd$ and $\tild y\in \chosenbody{\eta}$, let $\mynode'$ denote the node $\tild x \to \tild x_1 \to \cdots \to  \tild x_n\to \tild y$.
  We should show that $\branchsel_{\mynode'}$ is a branch selection for $y$. Now this follows easily from the fact that
  \begin{align*}
    \branchsel_{\mynode'} &= \{y\to\branch\mid  x \to  x_1 \to \cdots \to  x_n\to y \to \branch \in \branchsel\}\\
     &= \{y\to\branch \mid  x_n\to y\to \branch \in \branchsel_\mynode \}
  \end{align*}
  and this last set is guaranteed to be a branch selection (by \cref{cor:inductionstep}).
\end{itemize}

This way of constructing $J$ indeed results in a locally complete justification: for each internal node \mynode labeled $\tild z$ of $T$ that is part of $J$, a rule  $\tild z \gets B$ is selected and the children of $\mynode$ have labels in $B$, which thus concludes our proof.
\end{proof}

\section{Tree-Like Justification Systems are Consistent}\label{sec:result}

We now turn our attention to proving the main theorem of this paper (\cref{thm:main}).
% This section is devoted to proving the main theorem we just stated.
As before, we fix a justification system $\js=\jscomplete$ with $\jf=\jfcomplete$.

One direction of the main theorem is fairly easy to prove: if we are given a good justification for a certain fact, there cannot be a good justification for its complement.
This direction relies on the fact that for complementary frames, the cases of $x$ and those of $\tild x$, and the justification of $x$ and $\tild x$ are intrinsically related as formalized in the following two lemmas (inspired by similar results for graph-like justifications \cite{nmr/MarynissenPBD18}).

% In complementary justification systems, cases of a fact and its negation are intrinsically related.
\begin{lemma}\label{lem:complementaryintersectionoppositerules}
  If $\jf$ is complementary, then for all rules $x \gets A$ and $\tild x \gets B$ in $\rules$, we have $A \cap \tild B \neq \emptyset$.
\end{lemma}
\begin{proof}
  Take $A \in \jf(x)$ and $B \in \jf(\tild x)$.
  By complementarity, there exists a selection function $\sel$ for $\tild x$ such that $\tild \im(\sel) \subseteq A$.
  Therefore, $\tild \sel(B) \in A$.
  On the other hand, $\sel(B) \in B$; hence $\tild \sel(B) \in A \cap \tild B$.
\end{proof}

% This lemma actually expands to justifications.

\begin{lemma}\label{lem:complementaryintersectionoppositejustifications}
  Let $\jf=\jfcomplete$ be a complementary justification frame and $x \in \Fd$.
  If $J$ and $K$ are
%   locally complete %BB- added this to def of justifications
justifications in $\justifications(x)$ and $\justifications(\tild x)$ respectively, then there exists a $J$-branch $\branch$ starting in $x$ such that $\tild \branch$ is a $K$-branch starting in $\tild x$.
\end{lemma}
\begin{proof}
%   For simplicity sake, we only prove this for graph-like justifications.
%   The proof for tree-like justifications is similar.
  We incrementally define $J$-paths $\branch_i$ and $K$-paths $\branch_i^*$ of length $i$ such that $\branch_i = \tild \branch_i^*$.
  Define $\branch_1$ and $\branch_1^*$ as the node $x$ and $\tild x$ respectively.
  Now assume that we obtained $\branch_i$ and $\branch_i^*$.
  Let $y$ be the end node of $\branch_i$.
  If $y$ is not defined, so is $\tild y$ and then $\branch_i$ is the desired $J$-branch.
  So assume $y$ is defined.
  We want to find a fact $z$ such that $z$ is a child of $y$ in $J$ and $\tild z$ is a child of $\tild y$ in $K$.
  By using the rules for $y \gets A$ in $J$ and $\tild y \gets B$ in $K$ we can use \cref{lem:complementaryintersectionoppositerules} to obtain that $A \cap \tild B \neq \emptyset$ because $\jf$ is complementary.
  Choose a $z$ in $A \cap \tild B$.
  Then we construct $\branch_{i+1} = \branch_i \rightarrow z$ and $\branch_{i+1}^* = \branch_i^* \rightarrow \tild z$.
  Our required branch is then the limit of $\branch_i$ for $i$ going to infinity.
\end{proof}

From this lemma, one direction of consistency directly follows; this result was also already shown by Marynissen \etal \cite[Proposition 5.13]{tplp/MarynissenBD20}, but we include a direct proof to make the current paper self-contained.

\begin{proposition}\label{prop:easypartofconsistency}
  If $\js$ is complementary and  $\be$ respects negation, then
%   Then
  \begin{equation*}
    \suppvalue(x,\interp) \leqt \tild \suppvalue(\tild x, \interp)
  \end{equation*}
  for any $x \in \Fd$ and any $\F$-interpretation $\interp$.
\end{proposition}
\begin{proof}
  Take $x$ with $\suppvalue(x,\interp) = \ell$ for some $\ell \in \lf$.
  This means there is a justification $J$ such that $\jval_{\js}(J,x,\interp) = \ell$.
  Take a justification $K$ for $\tild x$.
  Therefore, by \cref{lem:complementaryintersectionoppositejustifications}, there is a $J$-branch $\branch$ starting in $x$ such that $\tild \branch$ is a $K$-branch starting in $\tild x$. We consider three cases:
  \begin{itemize}\item
  If $\ell=\Tr$, then $\interp(\be(\branch)) = \Tr$; hence $K$ has a branch that is evaluated to $\Fa$.
  Therefore, $\jval_{\js}(K,\tild x, \interp) = \Fa$.
  Since $K$ was taken arbitrarily, we have that $\suppvalue(\tild x, \interp) = \Fa$.

  \item If $\ell = \Un$, we need to prove that $\Un \geqt \suppvalue(\tild x, \interp)$.
  Similarly, every justification $K$ for $\tild x$ has a branch $\tild \branch$ starting in $\tild x$ such that $\branch$ is a $J$-branch and  $\interp(\be(\branch)) \geqt \Un$.
  This shows that $\interp(\be(\tild \branch)) \leqt \Un$.
  Therefore, $\suppvalue(\tild x, \interp) \leqt \Un$.

  \item For $\ell = \Fa$, the statement is trivial.\qedhere
  \end{itemize}
\end{proof}

The other direction of the consistency is completely novel and follows directly from the theory developed in \cref{sec:core}.
\begin{proposition}\label{prop:difficultpartofconsistency}
  If $\js$ is complementary and  $\be$ respects negation, then
%   Then
  \begin{equation*}
    \suppvalue(x,\interp) \geqt \tild \suppvalue(\tild x, \interp)
  \end{equation*}
  for any $x \in \Fd$ and any $\F$-interpretation $\interp$.
\end{proposition}
\begin{proof}
  Consider the branch selection
  \[\branchsel_{\tild x} = \left\{\tild x \to \branch \mid \be(\tild x\to \branch) \leqt \suppvalue(\tild x,\interp)\right\}.\]
  This set is a branch selection for $\tild x$, because each justification rooted in $\tild x$ must have at least one branch with a value at most $\suppvalue(\tild x, \interp)$.
  \cref{thm:core} then guarantees that a justification
  $J$ rooted in a node labelled $x$ exists such that every branch $\branch$ in $J$ starting in the root is an element of $\tild \branchsel_{\tild x}$; hence $\be(\tild \branch) \leqt \suppvalue(\tild x, \interp)$.
  Since $\be$ respects negation, we have $\be(\branch) \geqt \tild \suppvalue(\tild x, \interp)$ for every $\branch$ in $J$ starting in the root and thus that $\jval(J, x, \interp)\geqt \tild \suppvalue(\tild x, \interp)$.
  From this, it immediately follows that $\suppvalue(x,\interp) \geqt \tild \suppvalue(\tild x, \interp)$.
\end{proof}

\begin{theorem}[Main theorem, restated]
  Let $\js=\jscomplete$ be a justification system. If $\js$ is complementary and $\be$ respects negation, then $\js$ is tree-like consistent.
\end{theorem}
\begin{proof}
 Follows by combining \cref{prop:easypartofconsistency} with \cref{prop:difficultpartofconsistency}.
\end{proof}

\section{Conclusion}\label{sec:concl}

Consistency is an important property that relates the absence of an explanation of a fact to the existence of an explanation for its complement.
Our results do not have an impact on branch evaluations that have been used before, since for all of them, consistency has separately been proven.
However, we believe it will simplify future applications of (tree-like) justification theory, by removing the burden of this proof obligation.
Moreover, the general formulation of our results in \cref{sec:core}, in fact entails that the so-called \emph{flattening} of a justification system \cite{phd/Marynissen22} is complementary.
This notion of flattening is important in the context of \emph{nested} justification systems \cite{lpnmr/DeneckerBS15,tplp/MarynissenBDH22}.

While  we have now once and for all resolved the consistency question for tree-like justifications,
for graph-like justifications, this problem is still open.
However, in the PhD thesis of the first author \cite[Corollary 2.4.7]{phd/Marynissen22}, we show that a justification system is graph-like consistent if and only if it is tree-like consistent and every tree-like justification can be transformed into a graph-like justification\footnote{In graph-like justifications, each fact occurs as the label of at most one node. The difficulty in this transformation is dealing with tree-like justifications that make multiple choices for a given fact.} (a property that is there called \emph{graph-reducibility}).
The results we presented in this work now entail that, under very mild assumptions, graph-like consistency is equivalent to graph-reducibility, and hence indirectly, also contributes to the understanding of graph-like justification systems.

\bibliographystyle{eptcs}
% \bibliography{idp-latex/krrlib}

\begin{thebibliography}{10}
\providecommand{\bibitemdeclare}[2]{}
\providecommand{\surnamestart}{}
\providecommand{\surnameend}{}
\providecommand{\urlprefix}{Available at }
\providecommand{\url}[1]{\texttt{#1}}
\providecommand{\href}[2]{\texttt{#2}}
\providecommand{\urlalt}[2]{\href{#1}{#2}}
\providecommand{\doi}[1]{doi:\urlalt{http://dx.doi.org/#1}{#1}}
\providecommand{\eprint}[1]{arXiv:\urlalt{https://arxiv.org/abs/#1}{#1}}
\providecommand{\bibinfo}[2]{#2}

\bibitemdeclare{inproceedings}{ijcai/BogaertsW18}
\bibitem{ijcai/BogaertsW18}
\bibinfo{author}{Bart \surnamestart Bogaerts\surnameend} \&
  \bibinfo{author}{Antonius \surnamestart Weinzierl\surnameend}
  (\bibinfo{year}{2018}): \emph{\bibinfo{title}{Exploiting Justifications for
  Lazy Grounding of Answer Set Programs}}.
\newblock In \bibinfo{editor}{J{\'{e}}r{\^{o}}me \surnamestart
  Lang\surnameend}, editor: {\sl \bibinfo{booktitle}{Proceedings of the
  Twenty-Seventh International Joint Conference on Artificial Intelligence,
  {IJCAI} 2018, July 13-19, 2018, Stockholm, Sweden.}},
  \bibinfo{publisher}{ijcai.org}, pp. \bibinfo{pages}{1737--1745},
  \doi{10.24963/ijcai.2018/240}.

\bibitemdeclare{phdthesis}{DeneckerPhD93}
\bibitem{DeneckerPhD93}
\bibinfo{author}{Marc \surnamestart Denecker\surnameend}
  (\bibinfo{year}{1993}): \emph{\bibinfo{title}{Knowledge representation and
  reasoning in incomplete logic programming}}.
\newblock Ph.D. thesis, \bibinfo{school}{K.U.Leuven}, \bibinfo{address}{Leuven,
  Belgium}.

\bibitemdeclare{inproceedings}{lpnmr/DeneckerBS15}
\bibitem{lpnmr/DeneckerBS15}
\bibinfo{author}{Marc \surnamestart Denecker\surnameend},
  \bibinfo{author}{Gerhard \surnamestart Brewka\surnameend} \&
  \bibinfo{author}{Hannes \surnamestart Strass\surnameend}
  (\bibinfo{year}{2015}): \emph{\bibinfo{title}{A Formal Theory of
  Justifications}}.
\newblock In \bibinfo{editor}{Francesco \surnamestart Calimeri\surnameend},
  \bibinfo{editor}{Giovambattista \surnamestart Ianni\surnameend} \&
  \bibinfo{editor}{Miros{\l}aw \surnamestart Truszczy{\'n}ski\surnameend},
  editors: {\sl \bibinfo{booktitle}{Logic Programming and Nonmonotonic
  Reasoning - 13th International Conference, {LPNMR} 2015, Lexington, KY, USA,
  September 27-30, 2015. Proceedings}}, {\sl \bibinfo{series}{Lecture Notes in
  Computer Science}} \bibinfo{volume}{9345}, \bibinfo{publisher}{Springer}, pp.
  \bibinfo{pages}{250--264}, \doi{10.1007/978-3-319-23264-5\_{}22}.

\bibitemdeclare{article}{ai/Dung95}
\bibitem{ai/Dung95}
\bibinfo{author}{Phan~Minh \surnamestart Dung\surnameend}
  (\bibinfo{year}{1995}): \emph{\bibinfo{title}{On the acceptability of
  arguments and its fundamental role in nonmonotonic reasoning, logic
  programming and n-person games}}.
\newblock {\sl \bibinfo{journal}{Artif. Intell.}}
  \bibinfo{volume}{77}(\bibinfo{number}{2}), pp. \bibinfo{pages}{321 -- 357},
  \doi{10.1016/0004-3702(94)00041-X}.

\bibitemdeclare{inproceedings}{iclp/GebserKKS09}
\bibitem{iclp/GebserKKS09}
\bibinfo{author}{Martin \surnamestart Gebser\surnameend},
  \bibinfo{author}{Roland \surnamestart Kaminski\surnameend},
  \bibinfo{author}{Benjamin \surnamestart Kaufmann\surnameend} \&
  \bibinfo{author}{Torsten \surnamestart Schaub\surnameend}
  (\bibinfo{year}{2009}): \emph{\bibinfo{title}{On the Implementation of Weight
  Constraint Rules in Conflict-Driven {ASP} Solvers}}.
\newblock In \bibinfo{editor}{Patricia~M. \surnamestart Hill\surnameend} \&
  \bibinfo{editor}{David~Scott \surnamestart Warren\surnameend}, editors: {\sl
  \bibinfo{booktitle}{ICLP}}, {\sl \bibinfo{series}{LNCS}}
  \bibinfo{volume}{5649}, \bibinfo{publisher}{Springer}, pp.
  \bibinfo{pages}{250--264}, \doi{10.1007/978-3-642-02846-5}.

\bibitemdeclare{inproceedings}{iclp/GelfondL88}
\bibitem{iclp/GelfondL88}
\bibinfo{author}{Michael \surnamestart Gelfond\surnameend} \&
  \bibinfo{author}{Vladimir \surnamestart Lifschitz\surnameend}
  (\bibinfo{year}{1988}): \emph{\bibinfo{title}{The Stable Model Semantics for
  Logic Programming}}.
\newblock In \bibinfo{editor}{Robert~A. \surnamestart Kowalski\surnameend} \&
  \bibinfo{editor}{Kenneth~A. \surnamestart Bowen\surnameend}, editors: {\sl
  \bibinfo{booktitle}{ICLP/SLP}}, \bibinfo{publisher}{MIT Press}, pp.
  \bibinfo{pages}{1070--1080}.
\newblock
  \urlprefix\url{http://citeseer.ist.psu.edu/viewdoc/summary?doi=10.1.1.24.6050}.

\bibitemdeclare{inproceedings}{concur/GimbertZ05}
\bibitem{concur/GimbertZ05}
\bibinfo{author}{Hugo \surnamestart Gimbert\surnameend} \&
  \bibinfo{author}{Wieslaw \surnamestart Zielonka\surnameend}
  (\bibinfo{year}{2005}): \emph{\bibinfo{title}{Games Where You Can Play
  Optimally Without Any Memory}}.
\newblock In \bibinfo{editor}{Mart{\'{\i}}n \surnamestart Abadi\surnameend} \&
  \bibinfo{editor}{Luca \surnamestart de~Alfaro\surnameend}, editors: {\sl
  \bibinfo{booktitle}{{CONCUR} 2005 - Concurrency Theory, 16th International
  Conference, {CONCUR} 2005, San Francisco, CA, USA, August 23-26, 2005,
  Proceedings}}, {\sl \bibinfo{series}{Lecture Notes in Computer Science}}
  \bibinfo{volume}{3653}, \bibinfo{publisher}{Springer}, pp.
  \bibinfo{pages}{428--442}, \doi{10.1007/11539452\_33}.

\bibitemdeclare{article}{tplp/HouDD10}
\bibitem{tplp/HouDD10}
\bibinfo{author}{Ping \surnamestart Hou\surnameend}, \bibinfo{author}{Broes
  \surnamestart {De Cat}\surnameend} \& \bibinfo{author}{Marc \surnamestart
  Denecker\surnameend} (\bibinfo{year}{2010}): \emph{\bibinfo{title}{{FO(FD)}:
  Extending classical logic with rule-based fixpoint definitions}}.
\newblock {\sl \bibinfo{journal}{TPLP}}
  \bibinfo{volume}{10}(\bibinfo{number}{4-6}), pp. \bibinfo{pages}{581--596},
  \doi{10.1017/S1471068410000293}.

\bibitemdeclare{inproceedings}{vmcai/LapauwBD20}
\bibitem{vmcai/LapauwBD20}
\bibinfo{author}{Ruben \surnamestart Lapauw\surnameend},
  \bibinfo{author}{Maurice \surnamestart Bruynooghe\surnameend} \&
  \bibinfo{author}{Marc \surnamestart Denecker\surnameend}
  (\bibinfo{year}{2020}): \emph{\bibinfo{title}{Improving Parity Game Solvers
  with Justifications}}.
\newblock In \bibinfo{editor}{Dirk \surnamestart Beyer\surnameend} \&
  \bibinfo{editor}{Damien \surnamestart Zufferey\surnameend}, editors: {\sl
  \bibinfo{booktitle}{Verification, Model Checking, and Abstract Interpretation
  - 21st International Conference, {VMCAI} 2020, New Orleans, LA, USA, January
  16-21, 2020, Proceedings}}, {\sl \bibinfo{series}{Lecture Notes in Computer
  Science}} \bibinfo{volume}{11990}, \bibinfo{publisher}{Springer}, pp.
  \bibinfo{pages}{449--470}, \doi{10.1007/978-3-030-39322-9\_21}.

\bibitemdeclare{phdthesis}{phd/Marynissen22}
\bibitem{phd/Marynissen22}
\bibinfo{author}{Simon \surnamestart Marynissen\surnameend}
  (\bibinfo{year}{2022}): \emph{\bibinfo{title}{Advances in Justification
  Theory}}.
\newblock Ph.D. thesis, \bibinfo{school}{Department of Computer Science, KU
  Leuven}.
\newblock \urlprefix\url{https://lirias.kuleuven.be/3646147}.
\newblock \bibinfo{note}{Denecker, Marc and Bart Bogaerts (supervisors)}.

\bibitemdeclare{article}{tplp/MarynissenBD20}
\bibitem{tplp/MarynissenBD20}
\bibinfo{author}{Simon \surnamestart Marynissen\surnameend},
  \bibinfo{author}{Bart \surnamestart Bogaerts\surnameend} \&
  \bibinfo{author}{Marc \surnamestart Denecker\surnameend}
  (\bibinfo{year}{2020}): \emph{\bibinfo{title}{Exploiting Game Theory for
  Analysing Justifications}}.
\newblock {\sl \bibinfo{journal}{Theory Pract. Log. Program.}}
  \bibinfo{volume}{20}(\bibinfo{number}{6}), pp. \bibinfo{pages}{880--894},
  \doi{10.1017/S1471068420000186}.

\bibitemdeclare{inproceedings}{ijcai/MarynissenBD21}
\bibitem{ijcai/MarynissenBD21}
\bibinfo{author}{Simon \surnamestart Marynissen\surnameend},
  \bibinfo{author}{Bart \surnamestart Bogaerts\surnameend} \&
  \bibinfo{author}{Marc \surnamestart Denecker\surnameend}
  (\bibinfo{year}{2021}): \emph{\bibinfo{title}{On the Relation Between
  Approximation Fixpoint Theory and Justification Theory}}.
\newblock In \bibinfo{editor}{Zhi{-}Hua \surnamestart Zhou\surnameend}, editor:
  {\sl \bibinfo{booktitle}{Proceedings of the Thirtieth International Joint
  Conference on Artificial Intelligence, {IJCAI} 2021, Virtual Event /
  Montreal, Canada, 19-27 August 2021}}, \bibinfo{publisher}{ijcai.org}, pp.
  \bibinfo{pages}{1973--1980}, \doi{10.24963/ijcai.2021/272}.

\bibitemdeclare{article}{tplp/MarynissenBDH22}
\bibitem{tplp/MarynissenBDH22}
\bibinfo{author}{Simon \surnamestart Marynissen\surnameend},
  \bibinfo{author}{Bart \surnamestart Bogaerts\surnameend},
  \bibinfo{author}{Marc \surnamestart Denecker\surnameend} \&
  \bibinfo{author}{Jesse \surnamestart Heyninck\surnameend}
  (\bibinfo{year}{2022}): \emph{\bibinfo{title}{On Nested Justification
  Systems}}.
\newblock {\sl \bibinfo{journal}{Theory Pract. Log. Program.}}
  \bibinfo{volume}{22}.
\newblock \bibinfo{note}{To appear (Accepted for ICLP 2022 special issue in
  TPLP)}.

\bibitemdeclare{inproceedings}{nmr/MarynissenPBD18}
\bibitem{nmr/MarynissenPBD18}
\bibinfo{author}{Simon \surnamestart Marynissen\surnameend},
  \bibinfo{author}{Niko \surnamestart Passchyn\surnameend},
  \bibinfo{author}{Bart \surnamestart Bogaerts\surnameend} \&
  \bibinfo{author}{Marc \surnamestart Denecker\surnameend}
  (\bibinfo{year}{2018}): \emph{\bibinfo{title}{Consistency in Justification
  Theory}}.
\newblock In: {\sl \bibinfo{booktitle}{Proceedings of 17th International
  Workshop on Non-Monotonic Reasoning {(NMR} 2018), Tempe, Arizona, USA, Oct.
  27-29, 2018}}, \bibinfo{publisher}{AAAI Press 2018}, pp.
  \bibinfo{pages}{41--52}.
\newblock \urlprefix\url{http://www4.uma.pt/nmr2018/NMR2018Proceedings.pdf}.

\bibitemdeclare{article}{GelderRS91}
\bibitem{GelderRS91}
\bibinfo{author}{Allen \surnamestart {Van Gelder}\surnameend},
  \bibinfo{author}{Kenneth~A. \surnamestart Ross\surnameend} \&
  \bibinfo{author}{John~S. \surnamestart Schlipf\surnameend}
  (\bibinfo{year}{1991}): \emph{\bibinfo{title}{The Well-Founded Semantics for
  General Logic Programs}}.
\newblock {\sl \bibinfo{journal}{J. ACM}}
  \bibinfo{volume}{38}(\bibinfo{number}{3}), pp. \bibinfo{pages}{620--650},
  \doi{10.1145/116825.116838}.

\end{thebibliography}
% \input{0-tree-like-consistency.bbl}
%%%

%%%

\end{document}